\documentclass[letterpaper, 10 pt, conference]{ieeeconf}

\IEEEoverridecommandlockouts        
\usepackage{cite}
\usepackage{amsmath,amssymb,amsfonts}
\usepackage{algorithmic}
\usepackage{hyperref}
\usepackage{textcomp}
\usepackage{graphicx,color}
\usepackage{mathrsfs}
\usepackage[vlined,ruled]{algorithm2e}
\usepackage{subfigure}
\usepackage{url}
\usepackage{color}
\usepackage{dsfont}
\usepackage{bbm}
\usepackage{booktabs}
\usepackage{array}
\usepackage[table]{xcolor} 
\usepackage{yfonts}
\usepackage[normalem]{ulem}
\usepackage{cite}

\newtheorem{theorem}{Theorem}[section]
\newtheorem{lemma}[theorem]{Lemma}
\newtheorem{definition}{Definition}
\newtheorem{corollary}[theorem]{Corollary}

\newtheorem{remark}{Remark}

\newcommand{\setdef}[2]{\{#1 \; : \; #2\}}

\usepackage{tabularx}
\usepackage{multirow}
\usepackage{makecell}

\newcommand\aamsout{\bgroup\markoverwith{\textcolor{violet}{\rule[0.5ex]{2pt}{1pt}}}\ULon}

\newcommand{\real}{\mathbb{R}}

\newcommand{\transpose}{\mathsf{T}} 

\newcommand{\mc}{\mathcal}

\DeclareSymbolFont{bbold}{U}{bbold}{m}{n}
\DeclareSymbolFontAlphabet{\mathbbold}{bbold}

\newcommand\oprocendsymbol{\hbox{$\square$}}
\newcommand\oprocend{\relax\ifmmode\else\unskip\hfill\fi\oprocendsymbol}

\renewcommand{\theenumi}{(\roman{enumi})}

\newcommand*{\QEDA}{\hfill\ensuremath{\blacksquare}}%

\newcommand\norm[1]{\left\lVert#1\right\rVert}

\DeclareMathOperator{\diag}{diag}

\graphicspath{{img/}}
\begin{document}

\title{\bf A Fundamental Performance Limitation for Adversarial
  Classification}

\author{Abed AlRahman Al Makdah, Vaibhav Katewa, and Fabio Pasqualetti
  \thanks{This work was supported in part by ARO award 71603NSYIP. The
    authors are with the Departments of Department of Electrical and
    Computer Engineering and Mechanical Engineering at the University
    of California, Riverside,
    \{\href{mailto:aalm005@ucr.edu}{\texttt{abedam}},\href{mailto:vkatewa@engr.ucr.edu}{\texttt{vkatewa}},
    \href{mailto:fabiopas@engr.ucr.edu}{\texttt{fabiopas\}@engr.ucr.edu.}}}}
\maketitle

\renewcommand{\baselinestretch}{0.995}

\begin{abstract}
  Despite the widespread use of machine learning algorithms to solve
  problems of technological, economic, and social relevance, provable
  guarantees on the performance of these data-driven algorithms are
  critically lacking, especially when the data originates from
  unreliable sources and is transmitted over unprotected and easily
  accessible channels.  In this paper we take an important step to
  bridge this gap and formally show that, in a quest to optimize their
  accuracy, binary classification algorithms -- including those based
  on machine-learning techniques -- inevitably become more sensitive
  to adversarial manipulation of the data. Further, for a given class
  of algorithms with the same complexity (i.e., number of
  classification boundaries), the fundamental tradeoff curve between
  accuracy and sensitivity depends solely on the statistics of the
  data, and cannot be improved by tuning the algorithm.
\end{abstract}



\section{Introduction}\label{sec: introduction}
Artificial intelligence and machine learning algorithms, including
neural networks, are used widely in technological, social and economic
applications, such as computer vision \cite{YL-KK-CF:10,AK-IS-GEH:12},
speech recognition
\cite{GED-DY-LD-AA:12,GH-LD-DY-GED-AM-NJ-AS-VV-PN-TNS-BK:12}, and
malware detection \cite{GED-JWS-LD-DY:13}. While these algorithms
typically achieve high performance under nominal and well-modeled
conditions, they are also very sensitive to small, targeted, and
possibly malicious manipulations of the training and execution data
\cite{CS-WZ-IS-JB-DE-IG-RF:14}. The reasons for this unreliable
behavior are still largely unknown, thus motivating the critical need
for novel theories and tools to deploy robust, reliable, and safe
data-driven algorithms.

In this paper we formally reveal a fundamental and previously unknown
tradeoff between the accuracy of a binary classification algorithm and
its sensitivity to arbitrary manipulation of the data. In particular,
we cast a binary classification problem into an hypothesis testing
framework, parametrize classification algorithms -- including those
based on machine learning techniques -- using their decision
boundaries, and show that the accuracy of the algorithm can be
maximized only at the expenses of its sensitivity. This tradeoff,
which applies to general classification algorithms, depends on the
statistics of the data, and cannot be improved by simply tuning the
algorithm. Our theory explains quantitatively how simple
algorithms can outperform more complex implementations when operating
in adversarial environments.

\noindent \textbf{Related work:} Recent work has shown that
classification based on neural networks is vulnerable to adversarial
perturbations \cite{CS-WZ-IS-JB-DE-IG-RF:14, IJG-JS-CS:14}, and that
these perturbations are universal and affect a large number of
classification algorithms. While heuristic explanations of this phenomenon and heuristic techniques have been proposed, including adversarial learning \cite{IJG-JS-CS:14, DL-CM:05,
  NP-PM-SJ-MF-ZBC-AS:16, SMMD-AAF-PF:16, AK-IG-SB:16,AK-IG-SB-YD:18}, black-box \cite{DL-CM:05}, and gradient-based \cite{IJG-JS-CS:14,
  NP-PM-SJ-MF-ZBC-AS:16}, a fundamental analytical understanding of
the limitations of classification algorithms under adversarial
perturbations is critically lacking. We identify these limitations for
a binary classification problem in a Bayesian setting. While in simple
setting, our analysis formally shows that a fundamental tradeoff
exists between accuracy and sensitivity of any classification
algorithm, independently of the complexity of the algorithm. The
papers \cite{SMMD-AAF-PF:16,AR-JS-PL:18,EW-JZK:17} are also related to
this study, which derive methods to measure robustness of different
classifiers against adversarial perturbations and obtain guarantees
against bounded perturbations, as well as \cite{AK-IG-SB:16}, which
shows how adversarial training improves the classifier's performance
against adversarial perturbations while deteriorating its performance
under nominal conditions. Our approach provides rigorous mathematical
support to the empirical evidence obtained in these works.

\noindent \textbf{Contribution:} This paper features three main
contributions. First, we propose metrics to quantify the accuracy of a
classification algorithm and its sensitivity to arbitrary manipulation
of the data. We prove that, under a set of mild technical assumptions,
the accuracy of a classification algorithm can only be maximized at
the expenses of its sensitivity. Thus, a fundamental tradeoff exists
between the performance of a classification algorithm in nominal and
adversarial settings. While our results formally apply to binary
classification problems, we conjecture that this fundamental tradeoff
in fact applies to more general classification problems. Second, we
show that a tradeoff between accuracy and sensitivity exists for
different classes of classification algorithms, and that simpler
algorithms can sometimes outperform more complex one in adversarial
settings. Third, for a fixed complexity of the classification
algorithm, we numerically show that the accuracy versus robustness
tradeoff depends solely on the statistics of the data, and cannot be
arbitrarily improved by tuning the classification algorithm, including
using sophisticated adversarial learning techniques. Taken together,
our results suggest that performance and robustness of data-driven
algorithms are dictated by the properties of the data, and not by the
sophistication or intelligence of the algorithm.




\section{Problem setup and preliminary notions}\label{sec: setup}
To reveal a fundamental tradeoff between the accuracy of a
classification algorithm and its robustness against malicious data
manipulation, we consider a binary classification problem where the
objective is to decide whether a scalar observation $x \in \real$
belongs to one of the classes $\mc H_0$ and $\mc H_1$. We assume that
the distribution of the observations satisfy
\begin{align}\label{eq:hypotheses}
  \begin{aligned}
    &\mc H_0: x \sim f_0(x ;\theta_0), \text{ and } \mc H_1: x \sim
    f_1(x ;\theta_1),
  \end{aligned} \end{align} where $f_0(x;\theta_0)$ and
$f_1(x;\theta_1)$ are arbitrary, yet known, probability density
functions with parameters $\theta_0 \in \mathbb{R}^{m_0}$ and
$\theta_1 \in \mathbb{R}^{m_1}$, respectively. We assume that the
partial derivatives of $f_k$ with respect to $x$ and $\theta_k$ exist
and are continuous over the domain of the distributions, for $k=0,1$.
Let $p_0$ and $p_1$ denote the prior probabilities of the observations
belonging to the classes $\mc H_0$ and $\mc H_1$, respectively.
Different (machine learning) algorithms can be used to solve the above
binary classification problem. Yet, because of the binary nature of
the problem, any classification algorithm can be represented by a
suitable partition of the real line, and  it can be written~as
\begin{align}\label{eq:general classifier}
  \mathfrak{C}(x;y) =
  \begin{cases}
    \mc H_0, & x \in \mc R_0,\\
    \mc H_{1}, & x \in \mc R_1,\\
  \end{cases}
\end{align}
where\footnote{For simplicity and without affecting generality, we
  assume that $n$ is~even. Further, an alternative configuration of
  the classifier \eqref{eq:general classifier} assigns $\mc{H}_0$ and
  $\mc{H}_1$ to $\mc{R}_1$ and $\mc{R}_0$, respectively. However,
  because accuracy and sensitivity of the two configurations can be
  obtained from each other, we consider only the configuration in
  \eqref{eq:general classifier} without affecting the generality of
  our analysis.} $y = [y_i]$ denotes a set of boundary points, with
\mbox{$y_0 \le \dots \le y_{n+1}$}, $y_0 = -\infty$,
$y_{n+1} = \infty$, and \begin{align*}
  \mc R_0 &= \setdef{z}{ y_i < z < y_{i+1},  \text{ with } i = 0,2,\dots, n} ,\\
  \mc R_1 &= \setdef{z}{ y_i \le z \le y_{i+1}, \text{ with } i = 1,3,\dots, n-1} .
\end{align*}
We refer to \eqref{eq:general classifier} as general classifier. We measure the performance of a classification algorithm through its
\emph{accuracy}, that is, its probability of making a correct
classification.
\begin{definition}{\bf \emph{(Accuracy of a classifier)}}\label{def: accuracy}
  The accuracy of the classification algorithm $\mathfrak{C}(x;y)$
  is
  \begin{align}\label{eq:accuracy}
    \begin{split}
      \mc A(y;\theta) = p_0 &\mathbf{P} \left[ x \in \mc R_0 | \mc
        H_0 \right] + p_{1} \mathbf{P} \left[ x \in \mc R_1 | \mc
        H_{1} \right] ,
    \end{split}
  \end{align}
  where
  $\theta =
  [\theta_0^{\transpose} \; \theta_1^{\transpose}]^{\transpose}$ contains
  the distribution parameters. ~\oprocend
\end{definition}
Using Equation \eqref{eq:accuracy} and the distributions in \eqref{eq:hypotheses}, we obtain
\begin{align} \label{eq:accuracy1}
  \begin{split}
    \mc A(y;\theta)
    &= p_0\Bigg(\sum_{l=1}^n(-1)^{l+1}\int\limits_{-\infty}^{y_l}
    f_0(x;\theta_0)dx + 1\Bigg) \\
    &+ p_{1}\Bigg(\sum_{l=1}^n(-1)^{l}\int \limits_{-\infty}^{y_l}
    f_{1}(x;\theta_{1})dx \Bigg).
  \end{split}
\end{align}
Clearly, the accuracy of a classification algorithm depends on the
position of its boundaries, which can be selected to maximize the
accuracy of the classification algorithm. To this aim, let $L(x)$
denote the Likelihood Ratio defined as
\begin{align*}
  L(x) = \frac{p_1 f_1(x; \theta_1)}{p_0  f_0(x; \theta_0)} .
\end{align*}
The Maximum Likelihood (ML) classifier is
\begin{align}\label{eq: ML classifier}
  \mathfrak{C}_\text{ML} (x; \eta) =
  \begin{cases}
    \mc H_0, & L(x) < \eta,\\
    \mc H_1, & L (x) \ge \eta,\\
  \end{cases} 
\end{align} 
where the threshold $\eta > 0$ is a design
parameter that determines the boundary points and, thus, the accuracy
of the classifier. As a known result in statistical hypothesis testing
\cite{TAS-AAG:06}, the accuracy of the ML classifier with $\eta = 1$
is the largest among all possible classifiers. The value and the
number of boundary points of the ML classifier depend on the
distributions $f_0(x;\theta_0)$ and $f_1(x;\theta_1)$, the threshold
$\eta$, and the prior probabilities through the equation
\begin{align}\label{eq:likelihood ratio equality}
  p_1f_1(x;\theta_1)-\eta p_0f_0(x;\theta_0)=0 .
\end{align}

Another important class of classifiers is the class of linear
classifiers, which are less complex and often achieve a competitive
performance compared to nonlinear classifiers (see
\cite{GXY-CHH-CJL:12} for more details). In our setting, a linear
classifier consists of one decision boundary $y\in\mathbb{R}$, and is
given by \begin{align}\label{eq:linear
    classifier}
  \mathfrak{C}_\text{L}(x;y) =
  \begin{cases}
    \mc H_0, & x < y, \\
    \mc H_1, & x \geq y. \\
  \end{cases}
\end{align}
Following Definition \ref{def:
    accuracy}, the accuracy of $\mathfrak{C}_\text{L}$ is
\begin{align} \label{eq:accuracy_lin_class}
\mc A(y;\theta)&=p_0\!\!\int\limits_{-\infty}^{y}\!\!\!f_0(x;\theta_0)dx - p_1\!\!\int\limits_{-\infty}^{y}\!\!\!f_1(x;\theta_1)dx+ p_1.
\end{align}
The optimal boundary $y_\text{L}^{*}$ that maximizes $\mc A(y;\theta)$ is 
\begin{align}\label{eq:optimizaion_linear_classifier}
  \begin{aligned}
    y_\text{L}^{*}=\:\: &  \underset{y_i}{\arg \max}
    & & \mc A(y_i;\theta) \\
    &\quad  \text{s.t.}
    & & y_i \text{ is a solution of \eqref{eq:likelihood ratio equality} with $\eta=1$.}
  \end{aligned}
\end{align}

While the boundaries are difficult to compute for general
distributions, they can be computed explicitly when the observations
are Gaussian (see below). Let
$\mathcal{N}(x;\mu,\sigma) = \frac{1}{\sqrt{2\pi \sigma^2}} e^{
  -\frac{(x - \mu)^2 }{ 2\sigma^2}}$ be the p.d.f. of a normal random
variable with mean $\mu$ and variance $\sigma$, and
$Q(z)=\int_{-\infty}^z\frac{1}{\sqrt{2\pi}}e^{\frac{-x^2}{2}}dx$ the
c.d.f. of the standard normal distribution.

\begin{remark}{\bf \emph{(ML and linear classifiers for Gaussian
      distributions)}}\label{remark: Gaussian classifier}
  For the Gaussian distributions
  $f_i(x;\theta_i) = \mc{N}(x;\mu_i,\sigma_i)$, $i = 0,1$, the
  boundaries of ML classifier satisfy
  \begin{align}\label{eq:likelihood ratio equality gaussian}
   &a x^2 + b x + c = 0   \qquad  \text{where,}\\ \nonumber
   &a= 
        \frac{1}{2}\Bigg(\frac{1}{\sigma_0^2}-\frac{1}{\sigma_1^2}\Bigg),
        b=\Bigg(\frac{\mu_1}{\sigma_1^2}-\frac{\mu_0}{\sigma_0^2}\Bigg),
       \text{ and }\\ \nonumber
    &c =\log\bigg(\frac{\sigma_0}{\sigma_1}\bigg)+ \log\bigg(
        \frac{p_1}{p_0}
        \bigg)+\frac{\mu_0^2}{2
        \sigma_0^2} -\frac{\mu_1^2}{2\sigma_1^2} -\log(\eta).
  \end{align}
  Equation \eqref{eq:likelihood ratio equality gaussian} has at most
  two real solutions, implying that the ML classifier
  has at most two decision boundaries (see Fig. \ref{Fig:pdf gaussian
    diff cov}). The ML classifier with boundaries corresponding to the
  solutions of \eqref{eq:likelihood ratio equality gaussian} with
  $\eta = 1$ has maximum accuracy \cite{TAS-AAG:06}. The solution of
  \eqref{eq:likelihood ratio equality gaussian} which maximizes the
  accuracy in \eqref{eq:accuracy_lin_class} is the boundary for the
  optimal linear classifier. \oprocend 
\end{remark}

\begin{figure}[!t]
  \centering
  \includegraphics[width=\columnwidth,trim={0cm 0.4cm 0cm
    -0.9cm},clip]{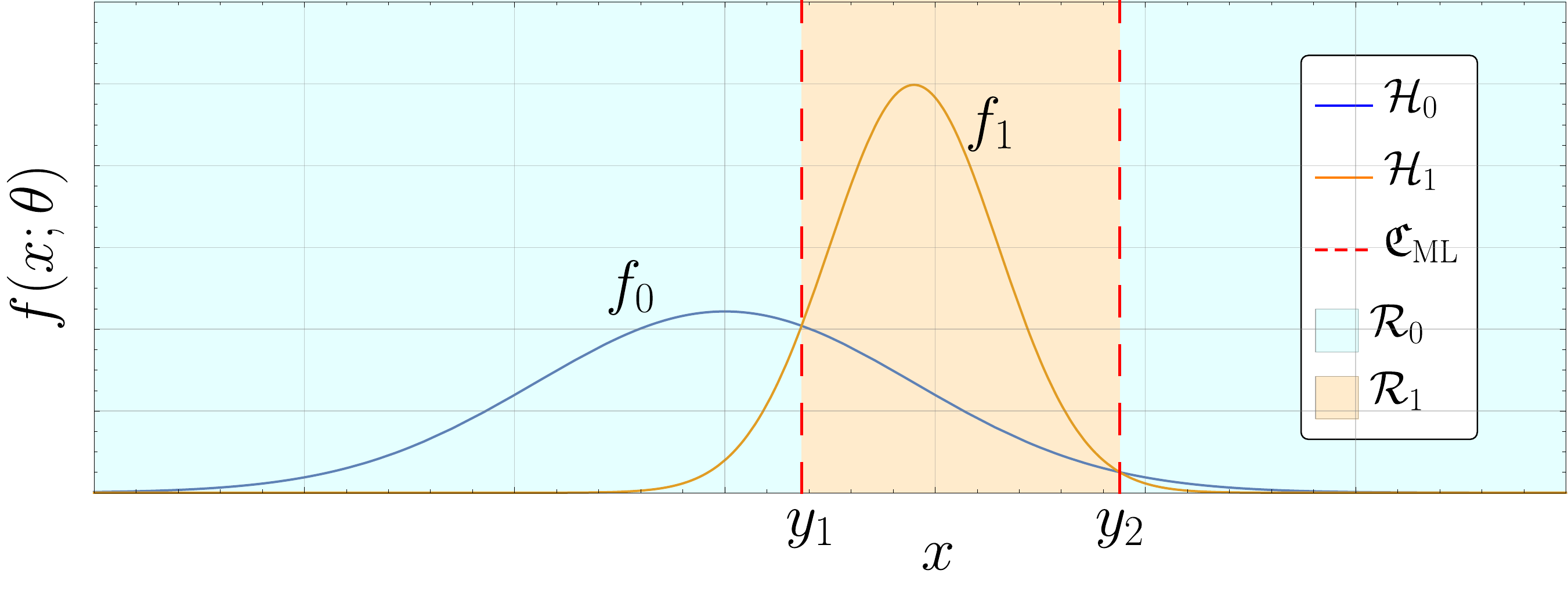}
  \vspace{-.5cm}
  \caption{The distributions of $x$ under Gaussian hypotheses with $\mu_0=0$, $\sigma_0=9$,
    $\mu_1=9$, $\sigma_1=4$, and $p_0=p_1=0.5$. The dashed red lines
    represent the decision boundaries of the
    $\mathfrak{C}_\text{ML}(x;\eta=1)$, which divide the space into
    $\mathcal{R}_0$ (represented by the blue region) and
    $\mathcal{R}_1$ (orange
    region).} 
  \label{Fig:pdf gaussian diff cov}
\end{figure}

To characterize the robustness of a classifier to adversarial
manipulation of the observations, we define the following sensitivity
metric, which capture the degradations of the classification accuracy
following data manipulation. It should be noticed that, by
manipulating the observations, the adversary effectively changes the
parameters of the distributions in \eqref{eq:hypotheses}.
\begin{definition}{\bf \emph{(Sensitivity of a
      classifier)}}\label{def: sensitivity}
  The sensitivity of the classification algorithm\footnote{Definition
    \ref{def: sensitivity} is also valid for the ML and the linear
    classifier.} $\mathfrak{C}(x;y)$ is
  \begin{align}\label{eq: sensitivity}
    \mc S(y; \theta) = \left\| \frac{\partial \mc A(y;\theta)}{ \partial
    \theta} \right\|_\infty ,
  \end{align}
  where $\theta$ contains the parameters of the distributions in
  \eqref{eq:hypotheses}, and $\mc A(y;\theta)$ denotes the accuracy
  of $\mathfrak{C}(x;y)$.  \oprocend
\end{definition}
From Definition \ref{def: sensitivity}, a higher value of sensitivity
implies that the adversary can affect the classifier's performance to
a larger extent, whereas a lower sensitivity implies that the
classifier is more robust to adversarial manipulation. Further, the
$\infty-$norm captures the worst case scenario in terms of the largest
sensitivity with respect to the components of $\theta$.

\begin{remark}{\bf \emph{(Accuracy and sensitivity of the ML
      classifier for Gaussian distributions)}}\label{remark: Gaussian
    classifier} The accuracy and the sensitivity of the ML classifier are obtained
  by substituting the expression of the normal distributions
  $\mc{N}(x;\mu_i,\sigma_i)$ in \eqref{eq:accuracy} and \eqref{eq:
    sensitivity}:
  \begin{align*}
    \begin{split}
      \mc A(y;\theta) &= p_0\Big(Q\Big(\frac{y_1 - \mu_0}{\sigma_0}\Big)
      - Q\Big(\frac{y_2-\mu_0}{\sigma_0}\Big)+1\Big)\\
      &+p_1\Big(-Q\Big(\frac{y_1-\mu_1}{\sigma_1} \Big) + Q
      \Big(\frac{y_2-\mu_1}{\sigma_1}\Big)\Big) \: \text{and,}
    \end{split}\\
    \begin{split} 
      \mc S(y;\theta) &= 
      \left\|
        \begin{bmatrix}
          p_0\Big(f_0\big(y_2;\theta_0 \big)-f_0(y_1;\theta_0)\Big)\\
          p_0\Big(\frac{\mu_0-y_1}{\sigma_0}f_0(y_1;\theta_0) -
          \frac{\mu_0-y_2}{\sigma_0}f_0(y_2;\theta_0)\Big)\\
          p_1\Big(f_1(y_1;\theta_1)-f_1(y_2;\theta_1)\Big)\\
          p_1\Big(\frac{\mu_1-y_2}{\sigma_1}f_1(y_2;\theta_1) -
          \frac{\mu_1-y_1}{\sigma_1}f_1(y_1;\theta_1)\Big)
        \end{bmatrix}
      \right\|_{\infty},
    \end{split}
  \end{align*}  
  where $\theta_i = [\mu_i \; \sigma_i ]^{\mathsf{T}}$ and $i=0,1$.  \oprocend
\end{remark}

A classification algorithm should be designed to have high accuracy
and low sensitivity, so as to exhibit robust satisfactory performance
in the face of adversarial manipulation. Unfortunately, in this paper
we show that accuracy and sensitivity are directly related, so that
optimizing the accuracy of a classifier inevitably also increases its
sensitivity.

\section{A fundamental tradeoff between accuracy and sensitivity of
  classification algorithms}\label{sec: tradeoff} 
In this section, we characterize the tradeoff between accuracy and
sensitivity of a classification algorithm for a given binary
classification problem as described in \eqref{eq:hypotheses}. In
particular, we prove that under some mild conditions, there exist a
classifier that is less accurate than $\mathfrak{C}_\text{ML}(x;1)$,
yet more robust to adversarial manipulation of the observations. This
shows that there exist a tradeoff between accuracy and sensitivity at
the the maximum accuracy configuration.

Let $y^* = [y_1^{*} \; y_2^{*} \; \cdots \; y_n^{*}]^{\transpose}$ be
the vector of the boundaries of $\mathfrak{C}_\text{ML}(x;1)$, which
maximizes $\mc A(y;\theta)$. Let $\theta^{(i)}$ be the $i^{\text{th}}$
component of $\theta$. We make the following assumptions:
\smallskip
\begin{itemize}
  \itemsep0.6em
\item[A1:] The vector
  $ \frac{\partial \mc A(y;\theta)}{\partial \theta} \Bigr|_{y^{*}}$
  has a unique largest absolute element, located at index $j$.
  
    
\item[A2:]There exist at least one boundary $y_i^*$ such that
\begin{align*}
\Bigg(p_0\frac{\partial }{\partial y_i} f_0(y_i;\theta_0) \Bigg|_{y_i^*}\!\!- p_{1}\frac{\partial}{\partial y_i} f_{1}(y_i;\theta_{1}) \Bigg|_{y_i^*}\Bigg) \frac{\partial y_i^*}{\partial \theta^{(j)}} \neq 0.
\end{align*}
\end{itemize}
Assumptions A1 is specific to our definition of sensitivity in \eqref{eq: sensitivity}, and is not required if $2-$norm is used (see Remark \ref{remark:
    assumption A1}). Further, A2 is mild and typically satisfied in
  most problems.

\begin{theorem}{\bf \emph{(Accuracy vs. sensitivity
      tradeoff for classifier \eqref{eq:general classifier})}}\label{thm:tradeoff}
  Let $y^*$ contain the boundaries of the classifier
  $\mathfrak{C}_\text{ML}(x;1)$. Then, under Assumptions A1 and A2, it holds
  \begin{align}
    \left.\frac{\partial{\mc S (y;\theta)}}{\partial
    y}\right|_{y^*}   \neq 0 .
  \end{align}
\end{theorem}
\begin{proof}
%
Assumption A1 guarantees that sensitivity $\mc{S}(y;\theta)$ is differentiable with respect to $y$ at $y^{*}$. Let $g\big(y;\theta\big)\triangleq \frac{\partial \mc A(y;\theta)}{\partial y}$. Since $y^*$ maximizes $\mc A(y;\theta)$, we have $g\big(y^*;\theta\big)=0$. Differentiating $g\big(y^*;\theta\big)$ with respect to $\theta^{(j)}$, and noting that $y^{*}$ depends on $\theta$, we get:
\begin{align}\label{pf: chain rule}
&\frac{\mathsf{d} g\big(y^*;\theta\big)}{\mathsf{d} \theta^{(j)}}=\frac{\partial g\big(y;\theta\big)}{\partial \theta^{(j)}}\Bigg|_{y^{*}}+\frac{\partial g\big(y;\theta\big)}{\partial y}\Bigg|_{y^{*}} \frac{\partial y^*}{\partial \theta^{(j)}} = 0, \nonumber \\
& \Rightarrow \frac{\partial}{\partial y}\frac{\partial \mc A(y;\theta)}{\theta^{(j)}}\Bigg|_{y^*}=-\frac{\partial^2 \mc A(y;\theta)}{\partial y^2}\Bigg|_{y^*}\frac{\partial y^*}{\partial \theta^{(j)}},
\end{align} 
where the last equation follows by substituting $g\big(y;\theta\big) =  \frac{\partial \mc A(y;\theta)}{\partial y}$ and switching the order of partial differentiation. Using \eqref{eq: sensitivity}, it can be easily observed that the left side of \eqref{pf: chain rule} equals $\pm \frac{\partial \mc S(y;\theta)}{\partial y}\Big|_{y^*}$. Further, differentiating \eqref{eq:accuracy1} twice, we get $\left. \frac{\partial^2}{\partial y^2} \mc A(y;\theta)=\diag(w_1(y_1),\cdots,w_n(y_n))\right.$, where
\begin{align*}
w_i(y_i) = p_0(-1)^{i+1}\frac{\partial}{\partial y_i}f_0(y_i;\theta_0)+ p_{1}(-1)^i\frac{\partial}{\partial y_i}f_{1}(y_i;\theta_{1}).
\end{align*}
Assumption A2 guarantees that there exist a boundary $y_i^{*}$ such that $w_i(y_i^{*})\frac{\partial y_i^*}{\partial \theta^{(j)}} \neq 0$. The reult follows from \eqref{pf: chain rule}.
\end{proof}

Theorem \ref{thm:tradeoff} implies that the sensitivity of the
classifier $\mathfrak{C}(x;y)$ can be decreased by modifying the
boundaries~$y^*$. Yet, because $\mathfrak{C}(x;y^{*})$ exhibits the
largest classification accuracy among all classifiers, the reduction
of sensitivity inevitably decreases the accuracy of classification. In
other words, for any classification problem \eqref{eq:hypotheses}
satisfying Assumption A1 and A2 and for any classification algorithm
\eqref{eq:general classifier}, there exists an arbitrarily small
$\delta$ such that\footnote{The inequality for accuracy is strict
  for most distributions.}
\begin{align*}
  \mc S (y^* + \delta; \theta) < \mc S (y^*; \theta)  \text{ and
  }  \mc A (y^* + \delta; \theta) \le \mc A(y^*; \theta) .
\end{align*}
This result also implies that the robustness of a classification
algorithm to adversarial manipulation of the data can be increased
only at the expense of its accuracy of classification. Thus, a
fundamental tradeoff exists between the accuracy of a classifier and
its robustness to adversarial manipulation. 
\begin{corollary}{\bf \emph{(Accuracy and sensitivity of the linear
      classifier \eqref{eq:linear classifier})}}\label{lemma:tradeoff
    linear classifier}
  Let $y_\text{L}^*$ be the boundary given in
  \eqref{eq:optimizaion_linear_classifier} that maximizes the accuracy
  (in \eqref{eq:accuracy_lin_class}) of the linear classifier
  $\mathfrak{C}_\text{L}(x;y)$. Then, under Assumptions A1 and A2, it holds
  \begin{align}
    \left.\frac{\partial{\mc S (y;\theta)}}{\partial
    y}\right|_{y_\text{L}^*}   \neq 0 .
  \end{align}
\end{corollary}
\begin{proof}
Since $y_\text{L}^{*}$ corresponds to one of the boundaries contained in $y^{*}$, the proof follows from Theorem~\ref{thm:tradeoff}.
\end{proof}

Next, we show that this tradeoff also exists for the Maximum
Likelihood classifier. This fact does not follow trivially from
Theorem \ref{thm:tradeoff}, because the general classifier in Theorem
has independent boundaries, while the boundaries of the Maximum
Likelihood are dependent from one another via \eqref{eq:likelihood
  ratio equality}. We make the following mild technical assumption.

\begin{itemize}
  \itemsep0.6em 
\item[A3:] The vectors
  $\left.\frac{\partial y(\eta,\theta)}{\partial
      \eta}\right|_{\eta=1}$ and
  $\left.\frac{\partial{\mc S (y;\theta)}}{\partial y}\right|_{y^*}$
  are not orthogonal, where $y(\eta,\theta)$ contains the boundaries
  of $\mathfrak{C}_\text{ML}(x;\eta)$. 
\end{itemize}

\begin{lemma}{\bf \emph{(Accuracy and sensitivity of the ML classifier \eqref{eq: ML classifier})}}\label{lemma:tradeoff ML}
  Let $y(\eta,\theta)$ contain the boundaries of the classifier
  $\mathfrak{C}_\text{ML}(x;\eta)$. Then, under Assumptions A1, A2 and
  A3, it holds
  \begin{align*}
    \left.\frac{\partial{\mc S \left(y(\eta,\theta);\theta\right)}}{\partial
    \eta}\right|_{\eta=1}   \neq 0 .
  \end{align*}
\end{lemma}
\begin{proof}
  Let $y^*$ contain the boundaries of the classifier
  $\mathfrak{C}_\text{ML}(x;\eta=1)$. The derivative of
  $\mc S \big(y(\eta,\theta);\theta\big)$ with respect to $\eta$ can
  be written as:
  \begin{align*}
  \left.\frac{\partial{\mc S \big(y(\eta,\theta);\theta\big)}}{\partial \eta}\right|_{\eta=1}
&=\left.\frac{\partial{\mc S \big(y;\theta\big)}}{\partial y^{\mathsf{T}}}\right|_{y^*}
\left.\frac{\partial y(\eta,\theta)}{\partial \eta}\right|_{\eta=1}.
\end{align*}
We conclude following Theorem \ref{thm:tradeoff} and Assumption~A3.
\end{proof}
%

%



\begin{figure*}[!h]
  \centering
  \includegraphics[width=\textwidth,trim={0.1cm 1.5cm 0.2cm
    0.74cm},clip]{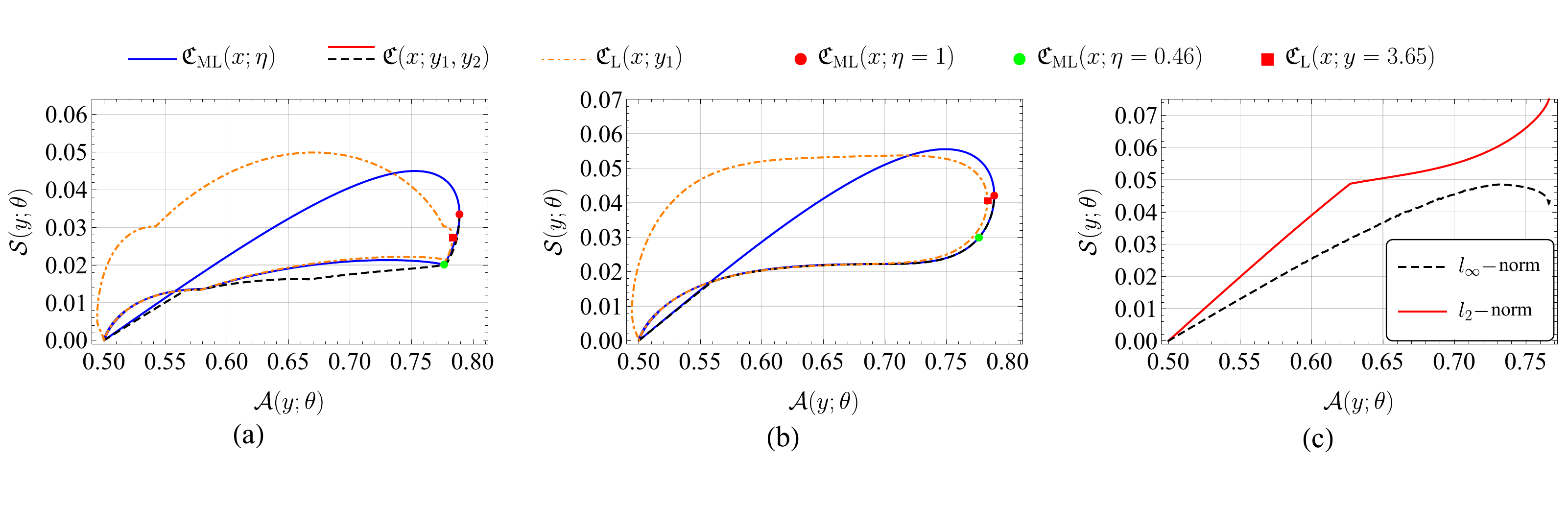}
  \vspace{-.5cm}
  \caption{Sensitivity-accuracy tradeoff curves for a general
    classifier with $2$ boundaries (black dashed line), the ML classifier
    (blue line), and a linear classifier (orange dash-dotted line) corresponding to the Gaussian hypothesis
    testing problem. The parameters of the two distributions for Fig.
    \ref{Fig:tradeoff curves}(a)-(b) are $\mu_0=0$, $\sigma_0=9$,
    $\mu_1=9$, and $\sigma_1=4$, and for Fig. \ref{Fig:tradeoff
      curves}(c) are $\mu_0=0$, $\sigma_0=4$, $\mu_1=5$, and
    $\sigma_1=3$. The red dot represents
    $\mathfrak{C}_{\text{ML}}(x;1)$ (maximum accuracy point) and the
    green dot represents $\mathfrak{C}_{\text{ML}}(x;0.46)$. The
    red square represents $\mathfrak{C}_{\text{L}}(x;y=3.65)$, which
    is the linear classifier with maximum accuracy. The sensitivity in
    \ref{Fig:tradeoff curves}(a) and \ref{Fig:tradeoff curves}(c) is
    computed using Definition \ref{def: sensitivity}, while the
    sensitivity in Fig. \ref{Fig:tradeoff curves}(b) is computed using
    \eqref{eq:sensitivity L2}.}
  \label{Fig:tradeoff curves}
\end{figure*}

In what follows we numerically show that a tradeoff between accuracy
and sensitivity also exists when the classification boundaries are not
selected to maximize the accuracy of the classifier. To this aim,
first we compute the accuracy and sensitivity of the ML classifier
$\mathfrak{C}_\text{ML}(x;\eta)$, for different values of $\eta$.
Notice that, by varying $0\le \eta < \infty$,
Equation~\eqref{eq:likelihood ratio equality} returns different
classification boundaries and, thus, different classification
algorithms. Similarly, we compute the accuracy and sensitivity of
linear classifier $\mathfrak{C}_\text{L}(x;y)$ by varying the single
boundary $y$. Second, we numerically solve
\begin{align}\label{eq:optimizaion}
  \begin{aligned}
    & \underset{y}{\text{min}}
    & & \mc S(y;\theta) \\
    & \text{s.t.}
    & & \mc A(y;\theta) = \zeta,
  \end{aligned}
\end{align}
for different values of $\zeta$ ranging from $0.5$ to
$\mc A(y^*; \theta)$. Notice that the minimization problem
\eqref{eq:optimizaion} returns the classifier with lowest sensitivity
and accuracy equal to $\zeta$, and that the boundaries solving the
minimization problem \eqref{eq:optimizaion} may not satisfy
\eqref{eq:likelihood ratio equality}. Further, for a given number of
classification boundaries, the minimization problem
\eqref{eq:optimizaion} returns a fundamental tradeoff curve relating
accuracy and sensitivity over the range of $\zeta$, which is
independent of the choice of classification algorithm. Finally, the
minimization problem \eqref{eq:optimizaion} is not convex, because of
its nonlinear equality constraint.

Fig. \ref{Fig:tradeoff curves}(a) shows the accuracy-sensitivity
tradeoff for the Gaussian hypothesis testing problem discussed in
Remark \ref{remark: Gaussian classifier}. In this case, since the ML
classifier has $2$ boundaries, we also consider general classifiers
with $2$ boundaries. We observe that the general classifier exhibits
the tradeoff at the maximum accuracy point (identified by the red dot)
in accordance with Theorem \ref{thm:tradeoff}. Several comments are in
order. First, the ML and linear classifiers also exhibit tradeoff at
their respective maximum accuracy points in accordance with Lemmas
\ref{lemma:tradeoff ML} and \ref{lemma:tradeoff linear classifier}.
Second, the tradeoff for the ML classifier is not strict and there
exist points where reducing accuracy increases sensitivity (green dot
in the figure). On the other hand, the tradeoff for the general
classifier is strict. This is because the decision boundaries of the
general classifier can be varied independently, whereas the boundaries
of the ML classifier are related to each other since they are the
solutions of \eqref{eq:likelihood ratio equality}. Thus, the general
classifier provides more flexibility in choosing the boundaries.
Similarly, the tradeoff for the linear classifier is not strict.
Third, the tradeoff curve for the general classifier is below the
tradeoff curves for the ML and linear classifier, again, due to the
aforementioned reason.\footnote{ML and linear
  classifiers are particular instances of the general~classifier.}
Fourth, the maximum accuracy of the linear classifier (corresponding
to red square) is smaller than that of ML classifier (corresponding to
the red dot), but its sensitivity at the maximum accuracy
configuration is also smaller than that of the ML classifier. This
explains the observed phenomena that in some cases, linear models are
more robust to adversarial attacks than nonlinear models (for example,
neural networks) \cite{AG-YV-XK:18}. Finally, the curves are not
smooth because of the $\infty$-norm in the definition \eqref{def:
  sensitivity}. 

We conclude with two remarks on using the $2$-norm to define
sensitivity and on the necessity of Assumption A1.

\begin{remark}{\bf \emph{ (Classification sensitivity using the
      $2-$norm)}} In Definition \ref{def: sensitivity}, the
  $\infty$-norm captures the largest change in accuracy with respect to a change in
  a single component of parameters vector $\theta$. Instead, using the
  $2$-norm to define the sensitivity of a classification algorithm
  leads to
  \begin{align}\label{eq:sensitivity L2}
    \mc S &(y;\theta) = \norm{\frac{\partial
            \mc A(y;\theta)}{\partial \theta}}_{2} ,
  \end{align}
  which captures the change in accuracy with respect to changes in all
  the components of $\theta$. Fig. \ref{Fig:tradeoff curves}(b) shows
  the sensitivity versus accuracy tradeoff when sensitivity is defined
  using \eqref{eq:sensitivity L2} instead of \eqref{eq: sensitivity}.
  Here, a strict tradeoff exists for the general, ML and linear
  classifiers. Further, the tradeoff curves are smooth since the
  $2$-norm is a smooth function. \oprocend 
\end{remark}

\begin{remark}{\bf \emph{(Necessity of Assumption A1)}}\label{remark:
    assumption A1}
  Theorem \ref{thm:tradeoff} may not hold when Assumption A1 is not
  satisfied, and we illustrate this fact in Fig. \ref{Fig:tradeoff
    curves}(c). In this case, the vector
  $\frac{\partial \mc A(y^{*};\theta)}{\partial \theta} = [0.043 \;
  0.024 \; -0.043 \; 0.040]^{\mathsf{T}}$ has two elements with
  maximum absolute value, violating Assumption A1. As a result, a
  tradeoff at the maximum accuracy point (denoted by red dot) does not
  exists as shown in the figure. Yet, a tradeoff still exist for
  sensitivity defined as in \eqref{eq:sensitivity L2}, indicating that
  A1 might be required only for definition \eqref{eq:
    sensitivity}. \oprocend
\end{remark}

\section{An illustrative example}\label{sec: example}
In this section we illustrate numerically the implications of Theorem
\ref{thm:tradeoff}. In particular, we consider two classification
algorithms with different accuracy and sensitivity, and show how their
performance degrades differently when the observations are corrupted
by an adversary. This implies that, when robustness to adversarial
manipulation of the observations is a concern, classification
algorithms should be designed to simultaneously optimize accuracy and
sensitivity, and should not operate at their point of maximum
accuracy.

Consider the classification problem \eqref{eq:hypotheses}, and let
\begin{align}\label{eq: nominal distributions}
  f_0 (x, \theta_0) = \mathcal{N}(x;\mu_0,\sigma_0), 
  f_1 (x, \theta_1) = \mathcal{N}(x;\mu_1,\sigma_1) .
\end{align}
Let $\mathfrak{C}^{1}=\mathfrak{C}_\text{ML}(x;1)$ and
$\mathfrak{C}^{2}=\mathfrak{C}_\text{ML}(x;0.4603)$ be the
classification algorithms identified by the red and green points in
Fig. \ref{Fig:tradeoff curves}(a), respectively. Notice that, when the
observations are not manipulated and follow the distributions
\eqref{eq: nominal distributions}, $\mathfrak{C}^{1}$ achieves higher
accuracy and sensitivity than $\mathfrak{C}^{2}$. This is also the
case when using definition \eqref{eq:sensitivity L2}, as illustrated
in Fig. \ref{Fig:tradeoff curves}(b). While the nominal distributions
\eqref{eq: nominal
  distributions} are used to design the classifiers $\mathfrak{C}^{1}$
and $\mathfrak{C}^{2}$, we consider an adversary that manipulates the
observations so that their true distributions~are 
\begin{align}\label{eq: attacked distributions}
  \begin{split}
    f_0 (x, \theta_0) &= \mathcal{N}(x;\mu_0 + \bar \mu_0,\sigma_0 + \bar
    \sigma_0),  \text{ and } \\
    f_1 (x, \theta_1) &= \mathcal{N}(x;\mu_1 + \bar
    \mu_1 ,\sigma_1 + \bar \sigma_1) ,
  \end{split}
\end{align}
where $\bar \mu_0$, $\bar \mu_1$, $\bar \sigma_0$, and $\bar \sigma_1$
are unknown parameters selected by the adversary to deteriorate the
accuracy of the classifiers.

\begin{table}[t!]
  \centering
  \caption{Numerical Results}
  \vspace{-0.1in}
  \renewcommand{\arraystretch}{1.4} 
  \begin{tabular}[c]{l|m{0.5cm}m{0.5cm}m{0.75cm}m{0.75cm}m{0.75cm}m{0.75cm}}
    \Xhline{2\arrayrulewidth}
    Classifier                              & $y_1$ & $y_2$  & $\mc S(y;\theta)$ & $\mc A(y;\theta)$ & $\mc A_{s1}(x)$ & $\mc A_{s2}(x)$\\
    \hline
    $\mathfrak{C}^{1}$ & 3.65 & 18.78  & 0.0334 & 0.7891 & 0.6857 & 0.6808\\
    $\mathfrak{C}^{2}$ & 1.83 & 20.60  & 0.0201 & 0.7766 & 0.6947 & 0.6939\\
    \Xhline{2\arrayrulewidth}
  \end{tabular}
  \label{Table:numerical results}
\end{table}

To evaluate the accuracy of $\mathfrak{C}^{1}$ and $\mathfrak{C}^{2}$
to classify the manipulated observations, we generate $10000$
observations obeying the modified distributions \eqref{eq: attacked
  distributions}, and compute the accuracy of the classifiers as the
ratio of the number of correct predictions to the total number of
observations. We repeat this experiment $100$ times, and then compute
the average accuracy of the classifiers over all trials.

Table \ref{Table:numerical results} summarizes the results of the
classification problems with $\mathfrak{C}^{1}$ and $\mathfrak{C}^{2}$
on the altered observations. 
In particular, $y_1$ and $y_2$ are the
decision boundaries of the classifiers, while $\mc{S}(y;\theta)$ and
$\mc{A}(y;\theta)$ denote their nominal sensitivity and accuracy. Instead,
$\mc{A}_{s1}(x)$ and $\mc{A}_{s2}(x)$ denote the average accuracy of the
classifiers when, respectively, the adversarial parameters are
$\bar \mu_1 = \bar \mu_0 = \bar \sigma_0 = 0$, $\bar \sigma_1 = 3$,
and $\bar \mu_0 = 1$, $\bar \sigma_0 =2$, $\bar \mu_1 =-2$,
$\bar \sigma_1 = 1.5$. The results show that, although
$\mathfrak{C}^{1}$ exhibits higher accuracy that $\mathfrak{C}^{2}$
when the observations follow the nominal distributions \eqref{eq:
  nominal distributions}, $\mathfrak{C}^{2}$ outperforms
$\mathfrak{C}^{1}$ in both adversarial scenarios, as supported by our
analysis.

\section{Dependency of Accuracy and Sensitivity on the parameters of
  the distributions} \label{sec:acc_sens_dist_param} In this section
we analyze the effect of the parameters
$\theta = [\theta_0 \; \theta_1]^{\transpose}$ on the accuracy and
sensitivity of the classifiers. We consider the Maximum Likelihood
classifier $\mathfrak{C}_\text{ML}(x;\eta= 1)$ for the analysis since
it maximizes the accuracy.\footnote{A similar analysis can also be
  performed for general and linear classifiers. However, we omit this
  analysis due to space constraints.} Specifically, we wish to
determine the distribution parameters that minimize the sensitivity
while providing a given level of accuracy. We consider the
following problem: 
\begin{align}\label{eq:optimizaion
    over parameters}
  \begin{aligned}
    & \underset{\theta \in \Theta}{\text{min}}
    & & \mc S\big(y^*(\theta),\theta\big) \\
    & \text{s.t.}
    & & \mc A\big(y^*(\theta),\theta\big) = \gamma,
    & & 
  \end{aligned}
\end{align}
where $y^*(\theta)$ denotes the boundaries of the ML classifier
$\mathfrak{C}_\text{ML}(x;1)$, which depend on $\theta$ via
\eqref{eq:likelihood ratio equality}, $0.5 \leq \gamma \leq 1$ denotes
the accuracy level, and $\Theta$ denotes the set of admissible
parameters $\theta$ of the distributions. The optimization problem
\eqref{eq:optimizaion over parameters} captures the fundamental limit
of sensitivity that can be achieved by a ML classifier with a desired
level of accuracy. Note that, similarly to \eqref{eq:optimizaion}, the
optimization problem in \eqref{eq:optimizaion over parameters} is not
convex due to the nonlinear equality constraint.

\begin{figure}[t]
  \centering
  \includegraphics[width=\columnwidth,trim={0cm 2.8cm 0cm
    2cm},clip]{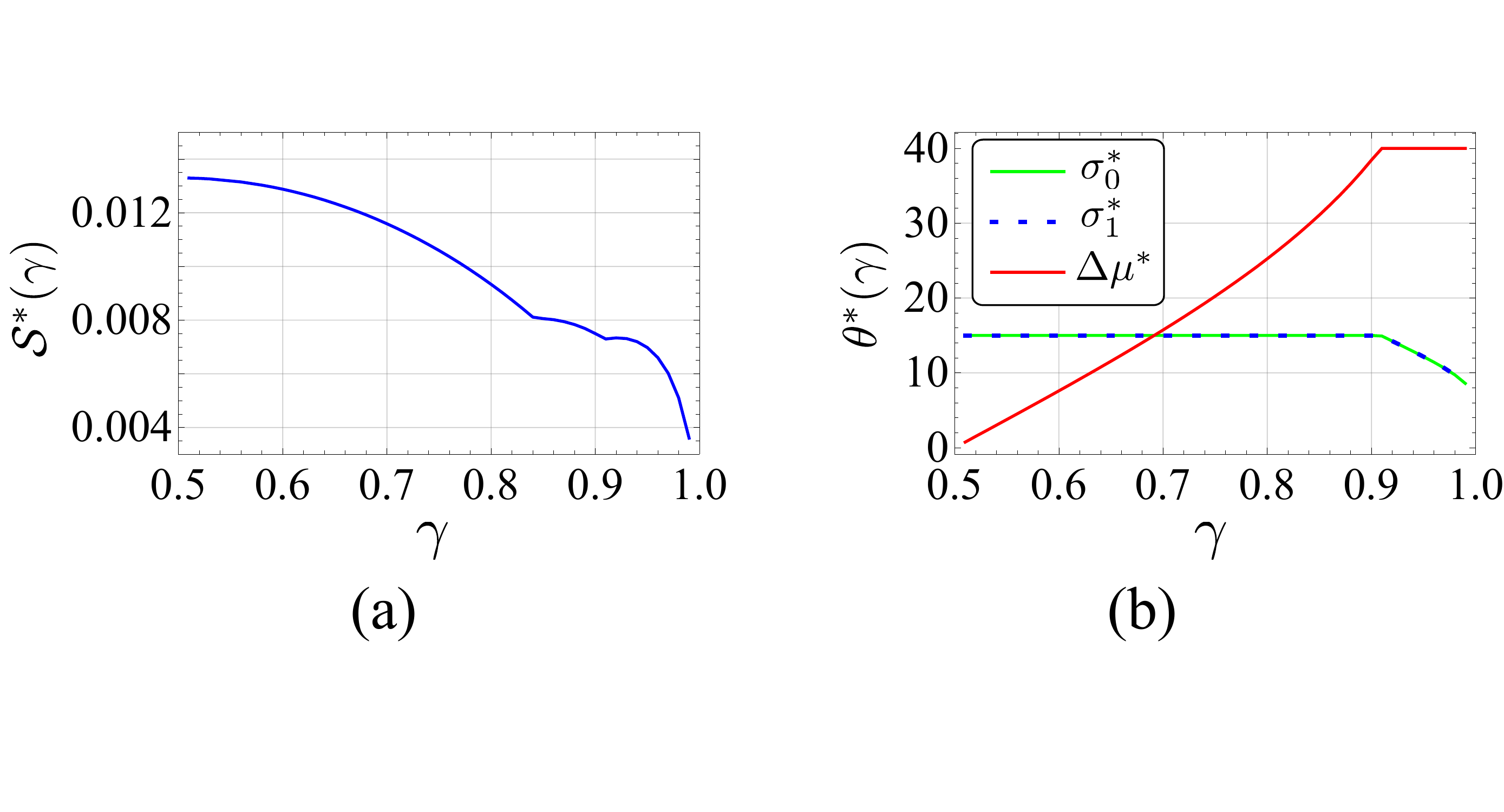}
  \vspace{-.5cm}
  \caption{(a) Minimum sensitivity (of \eqref{eq:optimizaion over
      parameters}) is a decreasing function of accuracy. (b)
    $\theta^{*}(\gamma) = [\mu^{*}_0, \sigma^{*}_0,\mu^{*}_1,
    \sigma^{*}_1]^{\transpose}$ denotes the optimal parameters of
    \eqref{eq:optimizaion over parameters}, with
    $\Delta \mu^{*} = |\mu^{*}_0-\mu^{*}_1|$. The constraints on
    $\theta$ are
    $|\mu_0-\mu_1| \leq 40, 0.1 \leq \sigma_1 \leq \sigma_0 \leq 15$.
    $\Delta \mu^{*}$ is non-decreasing and
    $\sigma^{*}_0, \sigma^{*}_1$ are non-increasing functions of
    $\gamma$, indicating that accuracy increases with the separation
    between the distributions $f_0$ and $f_1$.}
  \label{Fig:tradeoff curve wrt parameters}
\end{figure}

Let $\theta^{*}(\gamma)$ and $\mc{S}^{*}(\gamma)$ denote the optimal
parameters and minimum sensitivity of the optimization problem in
\eqref{eq:optimizaion over parameters}. Fig. \ref{Fig:tradeoff curve
  wrt parameters}(a) shows the variation of $\mc{S}^{*}(\gamma)$ as a
function of accuracy level $\gamma$ for the Gaussian hypothesis
testing problem detailed in Remark \ref{remark: Gaussian
  classifier}. It can be observed that $\mc{S}^{*}(\gamma)$ is a
decreasing function of $\gamma$. This is due to the fact that, to
achieve a higher level of accuracy, the ``separation'' between the two
distributions should be larger, as evident in Fig.~\ref{Fig:tradeoff
  curve wrt parameters}(b). At a larger separation, the effect of
changes in the distribution parameters on the accuracy of the
classifier is smaller, thereby resulting in a smaller sensitivity.

\begin{lemma} {\bf \emph{(Accuracy and sensitivity for Gaussian
      testing)}} Consider an hypothesis testing problem with
  $f_0 = \mc{N}(x;\mu_0,\sigma)$ and $f_1 = \mc{N}(x;\mu_1,\sigma)$,
  with $\theta = [\mu_0 \; \mu_1]^{\transpose}$ and $p_0 = p_1 = 0.5$.
  Assume that $\sigma$ is fixed. Then, for classifier
  $\mathfrak{C}_\text{ML}(x;1)$, $\mc{S}^{*}(\gamma)$ is a decreasing
  function of accuracy $\gamma$. \end{lemma} \begin{proof} For the Gaussian
  testing problem with $\sigma_0=\sigma_1 = \sigma$, $p_0=p_1=0.5$,
  Equation \eqref{eq:likelihood ratio equality} has a single solution
  for $\eta=1$ given by $y^*(\theta) = \frac{\mu_0 + \mu_1}{2}$. Using
  \eqref{eq:accuracy_lin_class}, the accuracy is given by
  $\mc{A}(y^*(\theta);\theta) =
  Q\left(\frac{|\mu_1-\mu_0|}{2\sigma}\right)$. Since $\sigma$ is
  fixed, we take the derivative of $\mc{A}(y^*(\theta);\theta)$ with
  respect to the means:
  \begin{align*}
    \mc{S}(y^*(\theta);\theta) \!&=\! 
                                   \frac{1}{2\sigma\sqrt{2\pi}}
                                   e^{-\frac{(\mu_1-\mu_0)^2}{2}}.
  \end{align*} 
  To conclude, $\mc{A}(y^*(\theta);\theta)$ and
  $\mc{S}(y^*(\theta);\theta)$ are increasing and decreasing functions
  of $|\mu_1-\mu_0|$, respectively. 
\end{proof}

\begin{lemma} {\bf \emph{(Accuracy and sensitivity for Exponential
      testing)}}
  Consider an hypothesis testing problem with
  $f_0(x;\lambda_0) = \lambda_0 e^{-\lambda_0 x}$ and
  $f_1(x;\lambda_1) = \lambda_1 e^{-\lambda_1 x}$, with $x\geq 0$,
  $\theta = \lambda_1$, and $p_0 = p_1 = 0.5$. Then, for
  classifier $\mathfrak{C}_\text{ML}(x;1)$ and a fixed $\lambda_0$,
  $\mc{S}^{*}(\gamma)$ is a decreasing function of accuracy~$\gamma$.
\end{lemma} \begin{proof}
  Without loss of generality, we assume $0<\lambda_0<\lambda_1$. For
  $p_0=p_1=0.5$, Equation \eqref{eq:likelihood ratio equality} has a
  single solution for $\eta=1$ given by
  $y^*(\theta) =\frac{1}{\lambda_1-\lambda_0}
  \log(\frac{\lambda_1}{\lambda_0})$. Using
  \eqref{eq:accuracy_lin_class}, 
 \begin{align*}
 \mc{A}(y^*(\theta);\theta) = 0.5 + 0.5 (r-1) r^{-\frac{r}{r-1}},
\end{align*}
where $r = \frac{\lambda_1}{\lambda_0}$. The sensitivity is given by
\begin{align*}
\mc{S}(y^*(\theta);\theta) &= \left \lvert \frac{\partial}{\partial r} \mc{A}(y^*(\theta);\theta) \right\rvert   \left \lvert \frac{\partial r}{\partial \lambda_1}  \right \lvert 
                             = \frac{\log(r)}{2\lambda_0 (r-1)} r^{-\frac{r}{r-1}}.
\end{align*}
To conclude, by inspecting the derivatives of $\mc{A}(y^*(\theta);\theta)$ and $\mc{S}(y^*(\theta);\theta)$ with respect to $r$, it can be seen that they are increasing and decreasing functions of $r$, respectively.
\end{proof}

\section{Conclusion and future work} \label{sec: conclusion} In this
paper we show that a fundamental tradeoff exists between the accuracy
of a binary classification algorithm and its sensitivity to
adversarial manipulation of the data. Thus, accuracy can only be
maximized at the expenses of the sensitivity to data manipulation, and
this tradeoff cannot be arbitrarily improved by tuning the algorithm's
parameters. Directions of future interest include the extension to
M-ary testing problems, as well as the formal characterization of the
relationships between the complexity of the classification algorithm
and its accuracy versus sensitivity tradeoff.






\renewcommand{\baselinestretch}{0.96}

\bibliographystyle{unsrt}
\bibliography{New}

\end{document}